\documentclass[10pt, conference, letterpaper]{IEEEtran}

\usepackage[utf8]{inputenc}   
\usepackage[T1]{fontenc}    	
\usepackage{url}              		 
\usepackage{booktabs}       	
\usepackage{amsfonts}       	
\usepackage{nicefrac}       	  
\usepackage{microtype}      	
\usepackage[american]{babel}   	

\usepackage{graphicx}
\usepackage{color}
\usepackage{rotating}
\usepackage{tabularx}
\usepackage{pdflscape}

\usepackage{amsmath,amsthm,amssymb}
\usepackage{algorithm, algorithmic}
\usepackage{bbm,dsfont}
\usepackage[caption=false]{subfig}
\usepackage{cancel}
\usepackage{enumerate, cases}
\usepackage{thmtools,thm-restate}
\usepackage{mathtools}

\usepackage[none]{hyphenat}
\usepackage{multirow}
\usepackage{makecell}
\usepackage{setspace}
\usepackage{pifont}
\usepackage{array}
\usepackage{enumitem}

\allowdisplaybreaks

\usepackage{silence}
\usepackage[capitalize]{cleveref}

\newcommand{\Regret}{\kR}

\newcommand{\EE}[1]{\bE\left[#1\right]}

\newcommand{\Var}[1]{\bV\left(#1\right)}

\newcommand{\R}{\bR}
\newcommand{\N}{\bN}

\renewcommand{\phi}{\varphi}
\renewcommand{\epsilon}{\varepsilon}

\newcommand{\al}[1]{ \begin{align} #1  \end{align}}
\newcommand{\eq}[1]{ \begin{equation} #1  \end{equation}}
\newcommand{\als}[1]{ \begin{align*} #1  \end{align*}}
\newcommand{\eqs}[1]{ \begin{equation*} #1  \end{equation*}}

\newcommand{\el}{\end{flushleft}}
\newcommand{\bl}{\begin{flushleft}}

\newcommand{\argmax}{\arg\!\max}

\newcommand{\bE}{\mathbb{E}}

\newcommand{\bN}{\mathbb{N}}

\newcommand{\bR}{\mathbb{R}}

\newcommand{\bV}{\mathbb{V}}

\newcommand{\cA}{\mathcal{A}}
\newcommand{\cB}{\mathcal{B}}

\newcommand{\cH}{\mathcal{H}}

\newcommand{\cN}{\mathcal{N}}

\newcommand{\cP}{\mathcal{P}}

\newcommand{\cV}{\mathcal{V}}

\newcommand{\kR}{\mathfrak{R}}

\theoremstyle{plain}

\newtheorem{rem}{Remark}

\begin{document}

	\title{Stochastic Multi-Armed Bandits with \\Limited Control Variates}
	
	\author{
       \IEEEauthorblockN{Arun Verma}
        \IEEEauthorblockA{
        Singapore-MIT Alliance for \\ Research and Technology Centre\\
        arun.verma@smart.mit.edu}
        \and
        \IEEEauthorblockN{Manjesh Kumar Hanawal}
        \IEEEauthorblockA{
        Department of IEOR\\
        Indian Institute of Technology Bombay\\
        mhanawal@iitb.ac.in}
        \and
        \IEEEauthorblockN{Arun Rajkumar}
        \IEEEauthorblockA{
        Department of DSAI\\
        Indian Institute of Technology Madras\\
        arunr@dsai.iitm.ac.in}
    }
	
	\maketitle
	
	\begin{abstract}
        Motivated by wireless networks where interference or channel state estimates provide partial insight into throughput, we study a variant of the classical stochastic multi-armed bandit problem in which the learner has limited access to auxiliary information. Recent work has shown that such auxiliary information, when available as control variates, can be used to get tighter confidence bounds, leading to lower regret. However, existing works assume that control variates are available in every round, which may not be realistic in several real-life scenarios. To address this, we propose UCB-LCV, an upper confidence bound (UCB) based algorithm that effectively combines the estimators obtained from rewards and control variates. When there is no control variate, UCB-LCV leads to a novel algorithm that we call UCB-NORMAL, outperforming its existing algorithms for the standard MAB setting with normally distributed rewards. Finally, we discuss variants of the proposed UCB-LCV that apply to general distributions and experimentally demonstrate that UCB-LCV outperforms existing bandit algorithms.
	\end{abstract}
	

	\IEEEpeerreviewmaketitle

	\section{Introduction}
	\label{sec:introduction}

We study the classical stochastic Multi-Armed Bandits problem (\cite{NOW12_bubeck2012regret, Book_lattimore2020bandit}) in the presence of additional information. Such information is naturally available in many real-life applications. For example, consider wireless networks: A transmitter can sense the channel to measure the interference level or estimate the channel state before transmitting the information in a given slot.
If measured interference at the transmitter is high, the receiver may also experience higher interference, resulting in a poor signal-to-noise ratio. Similarly, the receiver may experience poor signal reception if the channel state is poor. Both channel state and inference level impact the throughput at the receiver. 
Since both inference level and channel state information, or either of them, can act as additional information on the realized throughput, which is the metric of primary interest. However, channel state and inference level information may not always be available in each time slot, which can restrict the availability of additional information. 
In this work, we treat such additional information as side information and consider the scenario that such side information may only be limited, in the sense that it is not available in each time slot. Similar situations also arise in several other domains, such as recommending cabs with the least wait time to users, e-commerce platforms choosing sellers with the least delivery time to buyers (\cite{TNNLS22_santosh2022multiarmed, COMSNETS22_verma2022exploiting}), and job aggregating platforms \cite{NeurIPS21_verma2021stochastic}.

An interesting direction to model this side information is using the notion of control variates (CVs). Any observed random variable that is correlated with the reward can be considered as CV (\cite{OR82_lavenberg1982statistical, OR90_nelson1990control}).  
In the earlier wireless networks example, inference-level information and channel state information can act as CVs, as both provide additional information about the realized throughput.
Recent work \cite{NeurIPS21_verma2021stochastic} has exploited the availability of such CVs as side information to build estimators with smaller variance and develop an algorithm named UCB-CV that guarantees smaller regret. However, the work assumes that the CVs are available in every round. As seen in the examples above, this is not realistic in several real-world scenarios. Our goal in this paper is to develop an algorithm that exploits such additional but limited CV information.

We relax the assumption that CVs are available with each reward sample. In this setting, UCB-CV \cite{NeurIPS21_verma2021stochastic} is not applicable as CVs are not available in each round. We approach this problem by dividing the observations of rewards and CVs into two sets. The first set contains the observations where only rewards are observed, and the second set contains the observations where the rewards and associated CVs are observed. We use a standard empirical mean reward estimator for the observations in the first set, whereas we use a CV-based reward estimator for the observations in the second set. We propose an algorithm that combines these two estimators to get a better reward estimator. Specifically, our contributions are as follows:
\vspace{-2mm}
\begin{itemize}
	\itemsep0.2em
    \item For normally distributed rewards and CVs, we develop an algorithm, Upper Confidence Bounds with Limited Control Variates (\ref{alg:UCB-LCV}). It combines estimators from rewards samples with and without CVs (in \S~\ref{sec:normal_dist}). We also extended \ref{alg:UCB-LCV} for arms with multiple CVs.
   	
	\item Interestingly, a variant of \ref{alg:UCB-LCV} with no control variates, which we name as UCB-NORMAL, can be used for the stochastic MAB problem with normally distributed rewards. We show that UCB-NORMAL outperforms the classical UCB1-NORMAL algorithm of \cite{ML02_auer2002finite}. 
	
	\item In \S\ref{sec:no_dist}, we adapt \ref{alg:UCB-LCV} for problems where CVs can have non-Gaussian distributions using estimators and associated confidence intervals based on Jackknifing, splitting, and batching methods.
	
	\item We validate the different performance aspects of \ref{alg:UCB-LCV} on synthetically generated problems in \S\ref{sec:experiments}.
\end{itemize}

\subsection{Related Work}
Many bandit works use context as side information to learn the function that maps the optimal arm for given contexts. This bandit setting is popularly known as contextual bandits (\cite{COLT08_dani2008stochastic, MOR10_rusmevichientong2010linearly, WWW10_li2010contextual, NIPS11_abbasi2011improved, AISTATS11_chu2011contextual, NIPS10_filippi2010parametric,ICML17_li2017provably, NIPS17_jun2017scalable, ICML16_zhang2016online}) and has many real-life applications in areas like online advertising, online recommendations, and e-commerce. In contextual bandits, the mean reward of the arm is modeled as a function of the context, which is assumed to be observed \emph{before} playing any arm. Whereas, we consider an unstructured setup where side information available in the form of CVs is observed \emph{after} selecting an arm, and it does not alter the mean of arm rewards. It is just correlated with the reward samples. 

In the stochastic multi-armed bandit setup, the algorithms that exploit variance information are shown to be effective, especially when the arm rewards have unbounded support.    UCB1-NORMAL \cite{ML02_auer2002finite} uses the variance estimates in the index of the arms to achieve a regret of  $\mathcal{O}((\sigma^2/\Delta)\log T)$ when rewards are Gaussian with variance $\sigma^2$, where $\Delta$ is the smallest sub-optimality gap. \cite{ML02_auer2002finite} also experimentally demonstrate that the UCB-Tuned algorithm, which uses variance estimates in arm indices, significantly improves regret performance over UCB1 when rewards are bounded. Extending the idea of using `variance aware' indices, the UCBV \cite{TCS09_audibert2009exploration} algorithm shows that being variance aware is beneficial. EUCBV \cite{AAAI18_mukherjee2018efficient} improves the performance of UCBV using an arm elimination strategy as used in UCB-Improved \cite{PMH10_auer2010ucb}.

 CVs has been extensively for variance reduction (\cite{JORS85_james1985variance, Willy14_botev2014variance,ICML16_chen2016scalable,ACL17_kreutzer2017bandit,ICML19_vlassis2019design}) in the Monte-Carlo simulation of complex systems (\cite{OR82_lavenberg1982statistical,OR90_nelson1990control, MS82_lavenberg1981perspective,EJOR89_nelson1989batch}). \cite{NeurIPS21_verma2021stochastic} adopts these techniques in the stochastic MAB setting and proposes the UCB-CV algorithm assuming that CVs are observed in each round. Our work relaxes this assumption and considers the more practical scenario where CVs may not be available in all the rounds.

	\section{Problem Setting}		
	\label{sec:problem_setting}

We study the stochastic multi-armed bandits problem with control variates (MAB-CV), where the control variates may not be available for each reward sample observed. We follow the notations used in \cite{NeurIPS21_verma2021stochastic}. Let $K$ denote the number of arms and $[K] \doteq \{1, 2, \ldots, K\}$ the set of arms. In round $t$, the environment generates a vector $(\{X_{t,i}\}_{i \in [K]}, \{W_{t,i}\}_{i \in [K]})$. The variable  $X_{t, i}$ denotes the random reward of arm $i$, which is drawn from an unknown but fixed distribution with mean $\mu_i$ and variance $\sigma_i^2$. The reward observations $\left\{X_{s, i}\right\}_{s=1}^t$ form an Independent and Identically Distributed (IID) sequence. 
The random variable $W_{t, i}$ is the control variate (CV) associated with the reward of arm $i$.
It is independent of the rewards of all other arms in the same round, and is also independent of all rewards and control variates from other rounds.
These variables are drawn from an unknown but fixed distribution with mean $\omega_i$ and variance $\sigma_{w, i}^2$. The observations $\left\{W_{s, i}\right\}_{s=1}^t$ form an IID sequence. We assume that the learner knows the values $\{\omega_i\}_{i \in [K]}$\footnote{The control variates can be constructed such that their mean value is known (see Eq. (7) and Eq. (8) of \cite{ACL17_kreutzer2017bandit}). If the mean value is unknown, we estimate it from the samples, but using the estimated means can deteriorate the performance.} but not the $\{\sigma_{w, i}^2\}_{i \in [K]}$. The correlation coefficient between the rewards and their associated control variate for arm $i$ is denoted by $\rho_i$. 
In our problem, the learner observes the reward from the selected arm but may or may not observe the associated CVs. We refer to this new variant of the MAB-CV problem as Multi-Armed Bandits with Limited Control Variates (MAB-LCV). 
The parameter vectors $\boldsymbol{\mu^\sigma} = \{\mu_i, \sigma_i^2\}_{i \in [K]}$, $\boldsymbol{\omega^\sigma} = \{\omega_i, \sigma_{w,i}^2\}_{i \in [K]}$, and $\boldsymbol{\rho} = \{\rho_i\}_{i \in [K]}$ identify an instance of MAB-LCV problem, which is denoted $P = (\boldsymbol{\mu^\sigma}, \boldsymbol{\omega^\sigma}, \boldsymbol{\rho})$. The collections of all MAB-LCV problems are denoted by $\cP$. For an instance $P \in \cP$ with mean reward vector $\boldsymbol{\mu}$, we denote the optimal arm as $i^\star = \argmax_{i \in [K]} \mu_i$.

\begin{algorithm}[!ht]
    \renewcommand{\thealgorithm}{MAB-LCV problem}
	\floatname{algorithm}{}
	\caption{with instance $(\boldsymbol{\mu^\sigma}, \boldsymbol{\omega^\sigma}, \boldsymbol{\rho})$}
	\label{MABwLCV}
	\vspace{2mm}
	In each round $t$: 
	\begin{enumerate}
		\item \textbf{Environment} generates a vector $\boldsymbol{X_t} = (X_{t,1},\ldots,$ $X_{t,K}) \in \R^K$ and $\boldsymbol{W_t} = (W_{t,1},\ldots, W_{t,K}) \in \R^K$, where $\EE{X_{t,i}}=\mu_i$, Var$(X_{t,i}) = \sigma_i^2$, $\EE{W_{t,i}}=\omega_i$, Var$(W_{t,i}) = \sigma_{w,i}^2$, corr$(X_{t,i}, W_{t,i}) = \rho_i$, and  $(X_{t,i},W_{t,i})_{i\in [K]}$ are independent.
		\item \textbf{Learner} selects an arm $I_t \in [K]$ using the past observations of rewards and CV samples until round $t-1$. 
		\item \textbf{Feedback and Regret:} The learner observes a random reward $X_{t,I_t}$ but may or may not observe $W_{t,I_t}$. The learner incurs penalty of $(\mu_{i^\star} - \mu_{I_t})$.
	\end{enumerate}
\end{algorithm}

The interaction between the environment and learner is given in \ref{MABwLCV}. Since the vector $\boldsymbol{\mu}$ is unknown, we estimate $\boldsymbol{\mu}$ in an online fashion using the observations made at each round. Our goal is to learn a policy that collects the maximum reward and measures its performance by comparing its expected cumulative reward with an Oracle that plays the optimal arm in each round. We measure the performance of a policy that selects arm $I_t$ in round $t$ using the past observations of reward and associated control variate samples in terms of regret defined as follows:
$
	\Regret_T = T\mu_{i^\star} - \EE{\sum_{t=1}^{T} X_{t,I_t}}.
$

A policy is good if it has sub-linear expected regret, i.e., ${\Regret_T}/T \rightarrow 0$ as $T \rightarrow \infty$. 
To learn a policy with small sub-linear regret, we use available CVs to get estimated mean rewards with sharper confidence bounds. Therefore, the learner can have a better exploration versus exploitation trade-off and start playing the optimal arm sooner and more frequently.

	\section{Normally Distributed Rewards and CVs}
	\label{sec:normal_dist}

We first study the case where the rewards and associated CVs of arms follow a multivariate normal distribution, and later discuss general distributions in \S\ref{sec:no_dist}. 
To bring out the key ideas, we first consider the case where one control variate is associated with each arm, then extend to multiple CVs in \cref{ssec:multiCV}.
Next, we define an estimator that efficiently combines the rewards and available control variable samples.

\subsection{Mean Reward Estimator}
As the CV may not be available for each reward sample, we maintain two sets of indices for each arm. For arm $i\in [K]$, one set contains the indices of the observations where only reward samples are observed, and we denote it by $\cH_{i}(t)$. Another set contains the indices of observations where both reward and associate CV are observed and denoted by $\cH_{i}^c(t)$. We define an unbiased mean reward estimator for each set of observations.

Let $N_i(t) \in \N$ be the number of reward observations without CV for arm $i \in [K]$ at the beginning of round $t$. Then an unbiased estimator of the mean reward for reward observations in set $\cH_{i}(t)$ is given as follows:
\eq{
    \label{equ:estMeanReward}
	\hat\mu_{N_i(t),i}^{(nc)} = \frac{\sum_{n \in \cH_{i}(t)} X_{n,i}}{N_i(t)}.
}

Let $M_i(t) \in \N$ be the number of reward observations with associated CV for arm $i$ at the beginning of round $t$. We then use the following unbiased estimator \cite{NeurIPS21_verma2021stochastic} of mean reward for observations of reward and associated CV in set $\cH_{i}^c(t)$:
\eq{
    \label{equ:estMeanRewardCV}
	\hat\mu_{M_i(t),i}^{(c)} = \hat\mu_{M_i(t),i} + \hat\beta^\star_{M_i(t),i}\left(\omega_i - \hat\omega_{M_i(t),i}\right),
}
where $\hat\mu_{M_i(t),i} = {\sum_{m \in \cH_{i}^c(t)}X_{m,i}}/{M_i(t)}$, $\hat\omega_{M_i(t),i} = {\sum_{m \in \cH_{i}^c(t)}W_{m,i}}/{M_i(t)}$. The value of $\hat\beta^\star_{M_i(t),i}$ is 
\eq{
    \label{equ:BetaEst}
	\hat\beta^\star_{M_i(t),i}= \frac{\sum_{m \in \cH_{i}^c(t)} (X_{m,i} - \hat\mu_{M_i(t)})(W_{m,i} - \omega_i)}{\sum_{m \in \cH_{i}^c(t)}(W_{m,j}-\omega_m)^2},
}
which is motivated by the theory of Control Variates. It can be shown that this choice of $\hat\beta^\star_{M_i(t),i}$ minimizes the variance of $\hat\mu_{M_i(t)}^c$ while maintaining its unbiasedness property \cite{NeurIPS21_verma2021stochastic}.

Let $S_i(t) = N_i(t) + M_i(t)$ and $\lambda_{t,i} \in (0,1)$ $\forall t \ge 1$ and $i \in [K]$. Now we combine the unbiased mean reward estimators defined in Eq.~\ref{equ:estMeanReward} and Eq.~\ref{equ:estMeanRewardCV} as follows:
\eq{
    \label{equ:estCvxMeanReward}
	\hat\mu_{S_i(t),i} = \lambda_{t,i}\hat\mu_{N_i(t),i}^{(nc)} + (1-\lambda_{t,i})\hat\mu_{M_i(t),i}^{(c)}.
}
The following result discusses the different properties of the new estimator that will be useful to obtain a bound on its deviation form the true mean.
\begin{restatable}{prop}{estProp}
    \label{thm:estProp}
    Let $m=M_i(t)$, $n=N_i(t)$, and $s=S_i(t)$ for arm $i \in [K]$ at the beginning of round $t$. For any $\lambda_{t,i} \in (0,1)$, the mean reward estimator $\hat\mu_{s,i}$ has following properties:
    \begin{enumerate}
        \item $\hat\mu_{s,i}$ is an unbiased estimator, i.e., $\EE{\hat\mu_{s,i}} = \mu_i$,
        \item Variance of $\hat\mu_{s,i}$ is minimum for $\lambda_{t,i} = \frac{\Var{\hat\mu_{m,i}^{(c)}}}{{\Var{\hat\mu_{n,i}^{(nc)}} + \Var{\hat\mu_{m,i}^{(c)}}}}$, where $\Var{X}$ denotes the variance of random variable $X$, 
        \item $\Var{\hat\mu_{s,i}} =  \frac{\Var{\hat\mu_{n,i}^{(nc)}}\Var{\hat\mu_{m,i}^{(c)}}}{\Var{\hat\mu_{n,i}^{(nc)}} + \Var{\hat\mu_{m,i}^{(c)}}} = \frac{(1-\rho_i^2)\sigma_i^2}{m + n(1-\rho_i^2)}$, and
        \item If $\hat\mu_{m+n,i} = \frac{\sum_{m \in \cH_{i}^c(t)} X_{m,i} + \sum_{n \in \cH_{i}(t)} X_{n,i}}{m + n}$ is the mean reward estimator that does not use CV then $$\frac{\Var{\hat\mu_{s,i}}}{\Var{\hat\mu_{m+n,i}}} = \frac{(m+n)(1-\rho_i^2)}{m+n(1-\rho_i^2)}.$$
    \end{enumerate}
\end{restatable}

\begin{proof}
    Since $S_i(t) = M_i(t) + N_i(t)$, $s=m+n$. Consider the following combination of mean reward estimators defined in Eq.~\ref{equ:estMeanReward} and Eq.~\ref{equ:estMeanRewardCV} for an arm $i$ and $\lambda_{t,i} \in (0,1)$:
    \eq{
        \label{equ:cvxComEst}
    	\hat\mu_{s,i} = \lambda_{t,i}\hat\mu_{n,i}^{(nc)} + (1-\lambda_{t,i})\hat\mu_{m,i}^{(c)}.
    }
    
    \noindent
    \textbf{Proof of property 1.}
    Since $\hat\mu_{n,i}^{(nc)}$ and $\hat\mu_{m,i}^{(c)}$ are unbiased mean reward estimator and independent to each other, their combination defined in Eq.~\ref{equ:cvxComEst} is also an unbiased mean estimator, i.e., $\EE{\hat\mu_{s,i}} = \mu_{i}$ which follows from Theorem 1 of \cite{Biometrics59_graybill1959combining}. 
    
    \noindent
    \textbf{Proof of property 2.}
    The value of $\lambda_{t,i}$ is chosen such that the variance of the estimator $\hat\mu_{s,i}$ can be minimized. We first compute the variance of $\hat\mu_{s,i}$ for fix $\lambda_{t,i}$ and $\beta$ as follows:
    \als{
        \Var{\hat\mu_{s,i}} &= \Var{\lambda_{t,i}\hat\mu_{n,i}^{(nc)} + (1-\lambda_{t,i})\hat\mu_{m,i}^{(c)}} \\
        & = \Var{\lambda_{t,i}\hat\mu_{n,i}^{(nc)} + (1-\lambda_{t,i})(\hat\mu_{m,i} + \beta(\omega_{i} - \hat\omega_{m,i}))}
    }
    As $\hat\mu_{n,i}^{(nc)}$ and $\hat\mu_{m,i}^{(c)}$ are independent, we have
    \al{
        \implies &\Var{\hat\mu_{s,i}} = \lambda_{t,i}^2\Var{\hat\mu_{n,i}^{(nc)}} + (1-\lambda_{t,i})^2\big(\Var{\hat\mu_{m,i}} + \nonumber \\
        &\qquad\qquad\qquad \beta^2\Var{\hat\omega_{m,i}} - 2\beta\text{Cov}(\hat\mu_{m,i}, \hat\omega_{m,i})\big). \label{equ:varEq1} 
    }
    We first compute $\Var{\hat\mu_{n,i}^{(nc)}}$ as follows:
    \al{
       \Var{\hat\mu_{n,i}^{(nc)}} &= \Var{\frac{\sum_{i \in \cH_i(t)} X_i}{n}} = \frac{1}{n^2} \Var{\sum\nolimits_{i \in \cH_i(t)} X_i} \nonumber \\
       & = \frac{1}{n^2} \sum_{i \in \cH_i(t)}\Var{X_i} = \frac{1}{n^2} \sum_{i \in \cH_i(t)}\Var{X_i} \nonumber \\
       & = \frac{\Var{X_i}}{n}. \quad \left(\text{as $X_i$ are IID}\right) \label{equ:noCVVarEst}
    }   
    Similarly, we compute $\Var{\hat\mu_{m,i}}$ as follows:
    \als{
       \Var{\hat\mu_{m,i}} & = \Var{\frac{\sum_{i \in \cH_i^c(t)} X_i}{m}} = \frac{1}{m^2} \Var{\sum_{i \in \cH_i^c(t)} X_i} \\
       & = \frac{1}{m^2} \sum_{i \in \cH_i^c(t)}\Var{X_i} = \frac{\Var{X_i}}{m}. \quad \left(\text{as $X_i$ are IID}\right)
    } 
    We now compute $\Var{\hat\omega_{m,i}}$ as follows:
    \als{
       \Var{\hat\omega_{m,i}} &= \Var{\frac{\sum_{i \in \cH_i^c(t)} W_i}{m}} = \frac{1}{m^2} \Var{\sum_{i \in \cH_i^c(t)} W_i} \\
       & = \frac{1}{m^2} \sum_{i \in \cH_i^c(t)}\Var{W_i} = \frac{\Var{W_i}}{m}. \quad \left( \text{as $W_i$ are IID}\right)
    }
    Next we compute $\text{Cov}(\hat\mu_{m,i}, \hat\omega_{m,i})$ as follows:
    \als{
       \text{Cov}(\hat\mu_{m,i}, \hat\omega_{m,i}) &= \text{Cov}\left(\frac{\sum_{i \in \cH_i^c(t)} X_i}{m}, \frac{\sum_{i \in \cH_i^c(t)} W_i}{m} \right) \\
       & = \frac{1}{m^2} \text{Cov}\left(\sum\nolimits_{i \in \cH_i^c(t)} X_i, \sum\nolimits_{i \in \cH_i^c(t)} W_i \right).
    }
    As $X_i$ and $W_j$ are independent if $i \neq j$,
    \als{
       \text{Cov}(\hat\mu_{m,i}, \hat\omega_{m,i}) = \frac{1}{m^2} \hspace{-1mm}\sum_{i \in \cH_i^c(t)}\hspace{-1mm} \text{Cov}\left(X_i, W_i\right) = \frac{\text{Cov}\left(X_i, W_i\right)}{m}.
    }
    Putting the values of $\Var{\hat\mu_{n,i}^{(nc)}}$, $\Var{\hat\mu_{m,i}}$, $\Var{\hat\omega_{m,i}}$, and $\text{Cov}(\hat\mu_{m,i}, \hat\omega_{m,i})$ in Eq.~\ref{equ:varEq1}, we have
    \al{
        &\Var{\hat\mu_{s,i}} = \lambda_{t,i}^2 \frac{\Var{X_i}}{n} + (1-\lambda_{t,i})^2\bigg(\frac{\Var{X_i}}{m}\ + \nonumber\\
        &\qquad\qquad\qquad \beta^2\frac{\Var{W_i}}{m} - 2\beta\frac{\text{Cov}\left(X_i, W_i\right)}{m} \bigg). \label{equ:varEq2} 
    }
    Let us first find the optimal value of $\beta$ as follows:
    \al{
       &\frac{\partial \Var{\hat\mu_{s,i}}}{\partial \beta} = 0 \nonumber \\
       \implies& 0 + (1-\lambda_{t,i})^2\left(2\beta\frac{\Var{W_i}}{m} - 2\frac{\text{Cov}\left(X_i, W_i\right)}{m} \right) = 0 \nonumber \\
       \implies&2\beta\frac{\Var{W_i}}{m} = 2\frac{\text{Cov}\left(X_i, W_i\right)}{m} \nonumber \\
       \implies&\beta = \frac{\text{Cov}\left(X_i, W_i\right)}{\Var{W_i}}. \nonumber
    }
    Since $\lambda_{t,i} <1$ and $\Var{W_i} > 0$, we verify $\beta = \frac{\text{Cov}\left(X_i, W_i\right)}{\Var{W_i}}$ minimize $\Var{\hat\mu_{s,i}}$ as follows:
    \als{
       &\frac{\partial^2 \Var{\hat\mu_{s,i}}}{\partial \beta^2} = \frac{2(1-\lambda_{t,i})^2\Var{W_i}}{m} > 0. 
    }
    Let $A = \frac{\Var{X_i}}{n}$ be the $\Var{\hat\mu_{n,i}^{(nc)}}$ and $B=\frac{\Var{X_i}}{m} + \beta^2\frac{\Var{W_i}}{m} - 2\beta\frac{\text{Cov}\left(X_i, W_i\right)}{m}$ be the $\Var{\hat\mu_{m,i}^{(c)}}$ in Eq.~\ref{equ:varEq2}. Then we find the optimal value of $\lambda_{t,i}$ as follows:
    \als{
       &\frac{\partial \Var{\hat\mu_{s,i}}}{\partial \lambda_{t,i}} = 0 
       \implies 2\lambda_{t,i} A + 2(1-\lambda_{t,i}) (-1)B= 0 \\
       \implies&\lambda_{t,i} (A + B) - B = 0 
       \implies \lambda_{t,i}  = \frac{B}{A + B}.
    }
    As variance is positive, we verify $\lambda_{t,i} = \frac{B}{A+B}$ minimizes $\Var{\hat\mu_{s,i}}$ by checking
    $
       \frac{\partial^2 \Var{\hat\mu_{s,i}}}{\partial \lambda_{t,i}^2} = A+B > 0.
    $
    After replacing $A$ and $B$ by their values, the variance of $\hat\mu_{s,i}$ is minimum for $\lambda_{t,i} = \frac{\Var{\hat\mu_{m,i}^{(c)}}}{{\Var{\hat\mu_{n,i}^{(nc)}} + \Var{\hat\mu_{m,i}^{(c)}}}}.$
    
    \noindent
    \textbf{Proof of property 3.}
    The optimal value of $\beta$ does not depend on $\lambda_{t,i}$ and only the value of $B$ depends on $\beta$. Then we can find the minimum variance of the estimator as follows:
    \al{
        \Var{\hat\mu_{s,i}} &= \lambda_{t,i}^2 A + (1-\lambda_{t,i})^2 B \nonumber \\
        &=  \left(\frac{B}{A + B}\right)^2 A + \left( 1 -  \frac{B}{A + B} \right)^2 B \nonumber \\
        & = \frac{AB^2}{(A+B)^2} + \frac{A^2B}{(A+B)^2}  = \frac{AB^2 + A^2B}{(A+B)^2} \nonumber \\
        \implies \Var{\hat\mu_{s,i}} &= \frac{AB}{A+B}. \label{equ:varEq3}
    }
    Let $\rho_i$ be the correlation coefficient between $X_i$ and $W_i$, i.e, $\rho_i=\frac{\text{Cov}\left(X_i, W_i\right)}{\sqrt{\Var{X_i}\Var{W_i}}}$. The value of $B$ for $\beta = \frac{\text{Cov}\left(X_i, W_i\right)}{\Var{W_i}}$ is
    \als{
        B &=  \frac{\Var{X_i}}{m} + \beta^2\frac{\Var{W_i}}{m} - 2\beta\frac{\text{Cov}\left(X_i, W_i\right)}{m} \\
        &= \frac{1}{m} \bigg( \Var{X_i} + \left(\frac{\text{Cov}\left(X_i, W_i\right)}{\Var{W_i}} \right)^2\Var{W_i} \\
        &\qquad\qquad\qquad - 2\frac{\text{Cov}\left(X_i, W_i\right)}{\Var{W_i}}\text{Cov}\left(X_i, W_i\right) \bigg) \\
        &= \frac{1}{m} \left( \Var{X_i} + \frac{\text{Cov}\left(X_i, W_i\right)^2}{\Var{W_i}} - 2\frac{\text{Cov}\left(X_i, W_i\right)^2}{\Var{W_i}} \right) \\
        &= \frac{1}{m} \left( \Var{X_i} - \frac{\text{Cov}\left(X_i, W_i\right)^2}{\Var{X_i}\Var{W_i}} \Var{X_i}\right) \\
        &= \frac{1}{m} \left( \Var{X_i} - \rho_i^2\Var{X_i}\right) 
        = \frac{(1-\rho_i^2)\Var{X_i}}{m}.
    }
    Replacing $A= \frac{\Var{X_i}}{n}$ and $B=\frac{(1-\rho_i^2)\Var{X_i}}{m}$ in \ref{equ:varEq3}, we have
    \als{
        \Var{\hat\mu_{s,i}} &= \frac{AB}{A+B} = \frac{\frac{\Var{X_i}}{n}\frac{(1-\rho_i^2)\Var{X_i}}{m}}{\frac{\Var{X_i}}{n} + \frac{(1-\rho_i^2)\Var{X_i}}{m}} \\
        &= \frac{\frac{(1-\rho_i^2)\Var{X_i}}{nm}}{\frac{m + n(1-\rho_i^2)}{nm}} = \frac{(1-\rho_i^2)\Var{X_i}}{m + n(1-\rho_i^2)}. 
    }
    As $\sigma_i^2= \Var{X_i}$, we have $\Var{\hat\mu_{s,i}} = \frac{(1-\rho_i^2)\sigma_i^2}{m + n(1-\rho_i^2)}$.

	\noindent
    \textbf{Proof of property 4.}
    As $\Var{\hat\mu_{m+n,i}}$ is the variance of an estimator that does not use control variates. Then we have
    \als{
        \frac{\Var{\hat\mu_{s,i}}}{\Var{\hat\mu_{m+n,i}}} &= \frac{\frac{(1-\rho_i^2)\sigma_i^2}{m + n(1-\rho_i^2)}}{\Var{\frac{\sum_{i \in \cH_i(t)}  X_i + \sum_{j \in \cH_i^c(t)} X_j}{m+n}}} \\
        \intertext{As $X_i$ and $X_j$ are IID, and $\sigma_i^2 = \Var{X_i}$, we get}
        \frac{\Var{\hat\mu_{s,i}}}{\Var{\hat\mu_{m+n,i}}} & = \frac{\frac{(1-\rho_i^2)\sigma_i^2}{m + n(1-\rho_i^2)}}{\frac{\Var{X_i}}{m+n}} = \frac{(m+n)(1-\rho_i^2)}{m+n(1-\rho_i^2)}. \qedhere
    }
\end{proof}

The last property quantifies the reduction in variance of the estimator when it uses the available control variates information. 
The following result gives an unbiased variance estimator for the mean reward estimators defined in Eqs. \ref{equ:estMeanReward} and \ref{equ:estMeanRewardCV}, when the rewards and CVs follow a multivariate normal distribution.
\begin{restatable}{lem}{varEst}
	\label{lem:varEst}
	 Let the reward and control variate of an arm $i$ have a multivariate normal distribution. Define
	\als{
    	&A_{t,i} = \frac{\sum_{n \in \cH_{i}(t)} ({X}_{n,i} - \hat\mu_{N_i(t),i}^{(nc)})^2 }{N_i(t)(N_i(t)-1)}, \\
    	&Z_{t,i} = \left(1 - \frac{\left(\sum_{m \in \cH_{i}^c(t)} (W_{m,i} - \omega_i)\right)^2}{M_i(t)\sum_{m \in \cH_{i}^c(t)} (W_{m,i} - \omega_i)^2}\right)^{-1}, \\
    	&\bar{X}_{m,i} = X_{m,i} + \hat\beta^\star_{M_i(t),i}(\omega_i - W_{m,i}), 
    	\text{ and}\\
    	&B_{t,i} =  \frac{Z_{t,i}\sum_{m \in \cH_{i}^c(t)} (\bar{X}_{m,i} - \hat\mu_{M_i(t),i}^{(c)})^2}{M_i(t)(M_i(t)-2)}.
    }
	$A_{t,i}$ and $B_{t,i}$ are the unbiased variance estimator of $\hat\mu_{N_i(t),i}^{(nc)}$ and $\hat\mu_{M_i(t),i}^{(c)}$, respectively, i.e., $\EE{A_{t,i}}=\Var{\hat\mu_{N_i(t),i}^{(nc)}}$ and $\EE{B_{t,i}}=\Var{\hat\mu_{M_i(t),i}^{(c)}}$.
\end{restatable}

\begin{proof}
    The unbiased variance estimator of $\Var{X_i}$ using observations in $\cH_i(t)$ is $\Var{X_i} = \frac{\sum_{n \in \cH_{i}(t)} ({X}_{n,i} - \hat\mu_{N_i(t),i}^{(nc)})^2 }{N_i(t)-1}$. From Eq.~\ref{equ:noCVVarEst}, it is easy to follow that $A_{t,i}$ is unbiased estimator of $\Var{\hat\mu_{N_i(t),i}^{(nc)}}$. The $B_{t,i}$ is unbiased estimator  of $\Var{\hat\mu_{M_i(t),i}^{(c)}}$ follows from Lemma~1 of \cite{NeurIPS21_verma2021stochastic}.
\end{proof}

The value of $\lambda_{t,i}$ is estimated from the observed rewards and available control variate samples, i.e.,
$
    \hat\lambda_{t,i} = \frac{B_{t,i}}{A_{t,i}+B_{t,i}},
$
where $\cA_{t,i}$ and $\cB_{t,i}$ are same as defined in Lemma~\ref{lem:varEst}.
The variance of the mean reward estimators is inversely proportional to the number of observations they use (Lemma~\ref{lem:varEst}). The value of $\hat\lambda_{t,i}$ increases as $A_{t,i}$ decreases, and $A_{t,i}$ decreases as the number of observations without control variates ($N_i(t)$) increase. Therefore, as the fraction of $N_i(t)$ in $S_i(t)$ increases, the value of $\hat\lambda_{t,i}$ increases. By following the similar arguments, it is easy to see that the value of $\hat\lambda_{t,i}$ decreases as the fraction of observations with control variates ($M_i(t)$) in $S_i(t)$ increases. This way, the mean reward estimator (Eq.~\ref{equ:estCvxMeanReward}) gives more weights to the estimator that has a smaller variance. 

Our next result shows that using the estimated value of $\lambda_{t,i}$ in Eq.~\ref{equ:estCvxMeanReward} still leads to an unbiased reward estimator.
\begin{restatable}{lem}{rewardEst}
	\label{lem:rewardEst}
	Let the value of $\lambda_{t,i}$ be ${B_{t,i}}/\left({A_{t,i}+B_{t,i}}\right)$ in Eq.~\ref{equ:estCvxMeanReward}. Then, $\hat\mu_{S_i(t),i}^{(n)}$ is an unbiased mean reward estimator.
\end{restatable}
\begin{proof}
	For an unbiased estimator, we need to show that the expected value of the mean reward estimator is equal to the true value of the mean reward being estimated, i.e., $\EE{\hat\mu_{S,i}^{(n)}} = \mu$.  
	\als{
		&\EE{\hat\mu_{S,i}^{(n)}} =\EE{\hat\lambda_{t,i}\hat\mu_{N_i(t),i}^{(na)} + (1-\hat\lambda_{t,i})\hat\mu_{M_i(t),i}^{(a)}} \\
		&\hspace{1.5mm}=\EE{\hat\lambda_{t,i}\hat\mu_{N_i(t),i}^{(na)}} + \EE{(1-\hat\lambda_{t,i})\hat\mu_{M_i(t),i}^{(a)}} \\
		&\hspace{1.5mm}=\EE{\EE{\hat\lambda_{t,i}\hat\mu_{N_i(t),i}^{(na)}| \hat\lambda_{t,i}}} + \EE{\EE{(1-\hat\lambda_{t,i})\hat\mu_{M_i(t),i}^{(a)}|\hat\lambda_{t,i}}} \\
		&\hspace{1.5mm}=\EE{\hat\lambda_{t,i}\EE{\hat\mu_{N_i(t),i}^{(na)}| \hat\lambda_{t,i}}} + \EE{(1-\hat\lambda_{t,i})\EE{\hat\mu_{M_i(t),i}^{(a)}|\hat\lambda_{t,i}}}\hspace{-1mm}. \\
        \intertext{For IID samples, the sample mean and the sample variance are independent for a normal distribution. Therefore, we have}
		&\hspace{1.5mm}=\EE{\hat\lambda_{t,i}\EE{\hat\mu_{N_i(t),i}^{(na)}}} + \EE{(1-\hat\lambda_{t,i})\EE{\hat\mu_{M_i(t),i}^{(a)}}} \\
		&\hspace{1.5mm}=\EE{\hat\lambda_{t,i} \mu} + \EE{(1-\hat\lambda_{t,i})\mu} \\
		&\hspace{1.5mm}=\mu\EE{\hat\lambda_{t,i}} + \mu - \mu\EE{\hat\lambda_{t,i})} = \mu.   
    }
    \eqs{
        \implies \EE{\hat\mu_{S,i}^{(n)}} = \mu. \qedhere
    }
\end{proof}

Let $\nu_{S_i(t),i}=\Var{\hat{\mu}_{S_i(t),i}}$ be the variance of mean reward estimator $\hat{\mu}_{S_i(t),i}$. Similar to $\lambda_{t,i}$, the value of $\nu_{S_i(t),i}$ is estimated from the observed rewards and available control variate samples, i.e.,
\eq{
    \label{equ:varEst}
    \hat\nu_{S_i(t),i} = \frac{A_{t,i}B_{t,i}}{A_{t,i}+B_{t,i}}.
}
Since $\hat\mu_{N_i(t),i}^{(nc)}$ and $\hat{\mu}_{M_i(t),i}^c$ are independent estimators, we can show that $\hat\nu_{S_i(t),i} = \hat\lambda_{t,i}^2 A_{t,i} + (1-\hat\lambda_{t,i})^2 B_{t,i}$ is also an unbiased variance estimator for a given value of $\hat\lambda_{t,i}$ \cite{Biometrics59_graybill1959combining}. 
Equipped with these results, we next develop an \underline{u}pper \underline{c}onfidence \underline{b}ound (UCB) based algorithm for the MAB-LCV problem.

\subsection{Algorithm: \ref{alg:UCB-LCV}}
Recall that $S_i(t)$ denotes the number of times arm $i$ has been selected by the start of round $t$ and $\hat\nu_{S_i(t),i}$ be an unbiased sample variance of mean reward estimator $\hat\mu_{S_i(t),i}$ defined in Eq.~\ref{equ:estCvxMeanReward}. Motivated by UCB-CV \cite{NeurIPS21_verma2021stochastic}, we define an optimistic upper bound for the mean reward estimate of an arm $i$ as follows: 
\begin{equation}
	\label{equ:UCB}
	\text{UCB}_{t,i} = \hat\mu_{S_i(t),i} + \cV_{t,S_i(t)}^{(\alpha)}\sqrt{\hat\nu_{S_i(t),i}},
\end{equation}
where $\cV_{t,S_i(t)}^{(\alpha)}$ denote $100(1-1/t^\alpha)^{\text{th}}$ percentile value of the $t-$distribution with $S_i(t)-3$ degrees of freedom.
We use the $t-$distribution because the variance is unknown.
Using these UCB indices for arms, we develop an algorithm named \ref{alg:UCB-LCV} for the MAB-LCV problem. The algorithm works as follows: It takes the number of arms $K$, $\alpha>1$ (exploration and exploitation trade-off parameter), and $M=q+2$ (where $q$ is the number of CVs associated with the reward) as input.

\begin{algorithm}[!ht] 
	\renewcommand{\thealgorithm}{UCB-LCV}
	\floatname{algorithm}{}
	\caption{Algorithm for MAB-LCV problem}
	\label{alg:UCB-LCV}
	\begin{algorithmic}[1]
		\STATE \textbf{Input:} $K, ~\alpha>1,$ and $M$
		\STATE Play each arm $i \in [K]$ $Q=M+2$ times
		\STATE If $M_i(t) < M$ then set $\hat\lambda_{t,i}=1$
		\FOR{$t=QK+1, QK + 2, \ldots, $}
		\STATE $\forall i \in [K]:$ compute UCB$_{t,i}$ as given in \eqref{equ:UCB}
		\STATE Select $I_t = \argmax\limits_{i \in [K]}$ UCB$_{t,i}$
		\STATE Play arm $I_t$ and observe $X_{t,I_t}$
		\IF{$W_{t,It}$ is observed and $M_i(t) \ge M$}
		    \STATE Update $\hat\mu_{M_{I_t}(t),{I_t}}^{(c)}$and $B_{t,I_t}$
		\ELSE
		    \STATE Update $\hat\mu_{N_{I_t}(t),{I_t}}^{(nc)}$ and $A_{t,I_t}$
		\ENDIF
		\STATE Re-estimate $\hat\lambda_{t,I_t}$, $\hat\mu_{S_{I_t}(t),{I_t}}$, and $\hat\nu_{S_{I_t}(t), I_t}$
		\ENDFOR
	\end{algorithmic}
\end{algorithm}

\ref{alg:UCB-LCV} plays each arm $Q=M+2$ times to ensure that the sample variance for at least one of the mean reward estimators can be computed (see Lemma~\ref{lem:varEst}). After having $Q$ observations from arm $i$, if $M_i(t) < M$, then the value of $\hat\lambda_{t,i}$ is set to $1$ in the round $t$. Note that the value of $\hat\lambda_{t,i}$ is $0$ for $N_i(t) < 2$. Since we are playing each arm at least $Q$ times, $M_i(t) < M$ and $N_i(t) < 2$ can not be simultaneously true.

In the round $t$, \ref{alg:UCB-LCV} computes the upper confidence bound for the mean reward estimate of each arm using \eqref{equ:UCB} and then selects the arm having the highest UCB value. The selected arm is denoted by $I_t$. After playing arm $I_t$, the reward $X_{t, I_t}$ is observed, and the associated control variate $W_{t, I_t}$ may or may not be observed. If the control variate $W_{t, I_t}$ is observed and $M_i(t) \ge M$ then $\hat\mu_{M_{I_t}(t),{I_t}}^{(c)}$ and $B_{t,I_t}$ are updated; otherwise, $\hat\mu_{N_{I_t}(t),{I_t}}^{(nc)}$ and $A_{t,I_t}$ are updated. Accordingly, the values of $\hat\lambda_{t,I_t}$, $\hat\mu_{S_{I_t}(t),{I_t}}$, and $\hat\nu_{S_{I_t}(t), I_t}$ are re-estimated. A similar process is repeated for the subsequent rounds.

\subsection{Estimator with Multiple Control Variates}
\label{ssec:multiCV}
We now consider the case where each arm can have multiple CVs. Since only the mean reward estimator given in \eqref{equ:estMeanRewardCV} involves control variates, we will adopt this mean reward estimator to the multiple control variates case. We denote the number of CVs by $q$. Let $W_{m, i, j}$ be the $j^{\text{th}}$ control variate of arm $i$ for $m \in \cH_{i}^c(t)$. Then the unbiased mean reward estimator for an arm $i$ with multiple CVs is given as follows:
\eq{
	\label{equ:estMeanRewardMultiCV}
	\hat\mu_{M_i(t),i,q}^{(c)} = \hat\mu_{M_i(t),i} + \boldsymbol{\hat\beta}_{M_i(t),i}^{\star\top} (\boldsymbol{\omega}_{i} - \boldsymbol{\hat\omega}_{M_i(t),i}),
}
where $\boldsymbol{\hat\beta}^\star_{M_i(t),i} = \left(\hat\beta^\star_{M_i(t),i,1}, \ldots, \hat\beta^\star_{M_i(t),q}\right)^\top$\hspace{-2mm}, $\boldsymbol{\omega}_{i} = \left(\omega_{i,1}, \right.$ $ \left.\ldots, \omega_{i,q}\right)^\top$\hspace{-2mm}, and $\boldsymbol{\hat\omega}_{M_i(t),i}= \left(\hat\omega_{M_i(t),i,1}, \ldots, \hat\omega_{M_i(t),i,q}\right)^\top$, where $\hat{\omega}_{M_i(t),i,j} = \frac{1}{M_i(t)} \sum_{m \in \cH_{i}^c(t)} W_{m,i,j}$.

Let $\boldsymbol{W}_i$ be the $M_i(t)\times q$ matrix whose $m^{\text{th}}$ row is $\left(W_{m,i,1}, \ldots, W_{m,i,q} \right)$, and $\boldsymbol{X}_i=\left(\left\{X_{m,i} \right\}_{m \in \cH_i^c(t)}\right)^\top$. Then the optimal vector $\boldsymbol{\hat\beta}_{M_i(t),i}^{\star}$ (\cite{NeurIPS21_verma2021stochastic,OR90_nelson1990control}) is given by 
\eqs{
    \boldsymbol{\hat\beta}_{M_i(t),i}^{\star} = \frac{\boldsymbol{W}_i^\top\boldsymbol{W}_i - M_i(t)\boldsymbol{\hat\omega}_{M_i(t),i}\boldsymbol{\hat\omega}_{M_i(t),i}^\top}{\boldsymbol{W}_i^\top \boldsymbol{X}_i - M_i(t){\boldsymbol{\hat\omega}_i}~\hat\mu_{M_i(t),i}}.
}

We can easily generalize Lemma~\ref{lem:varEst} to get an unbiased variance estimator for mean reward estimator $\hat\mu_{M_i(t),i,q}^{(c)}$
(Lemma~4 of \cite{NeurIPS21_verma2021stochastic}). 
With this, we can now use \ref{alg:UCB-LCV} with $M=q+2$ for the MAB-LCV problem having multiple control variates. Let $\hat\mu_{S_i(t),i,q}$ be the combined estimator as defined in Lemma~\ref{lem:rewardEst} except it uses $\hat\mu_{M_i(t),i,q}^{(c)}$ instead of $\hat\mu_{M_i(t),i}^{(c)}$. Then in \ref{alg:UCB-LCV}, we use the following optimistic UCB index for an arm $i$ in-place of $\text{UCB}_{t,i}$:
\eqs{
	\text{UCB}_{t,i,q} = \hat\mu_{S_i(t),i,q} + \cV_{t,S_i(t),q}^{(\alpha)}\sqrt{\hat\nu_{S_i(t),i,q}},
}
where $\cV_{t,S_i(t),q}^{(\alpha)}$ is the $100(1-1/t^\alpha)^{\text{th}}$ percentile value of the $t-$distribution with $S_i(t)-q-2$ degrees of freedom and $\hat\nu_{S_i(t),i,q}$ is the sample variance of $\hat\mu_{S_i(t),i,q}$.

	\section{Special Cases of MAB-LCV}
	\label{sec:special_cases}

We will consider the following special cases of MAB-LCV problems in this section: MAB-LCV problems having no CVs, having CVs for all reward samples, and having CVs with probability $\epsilon$. 

\subsection{MAB-LCV problems having no CVs}
In this case, there are no control variates at all. This MAB-LCV problem is the same as the stochastic MAB problem with normally distributed rewards. If there are no control variates, then the value of $\hat\lambda_{t,i}$ will be $1$ for all $t\ge 1$ and $i \in [K]$, and $N_i(t)$ will be the same as $S_i(t)$. We refer to this variant of \ref{alg:UCB-LCV} algorithm as UCB-NORMAL. Note that UCB-CV [5] without CVs is not the same as UCB-NORMALbecause it uses the upper confidence bound that assumes the CV is available in every round.
Our next result gives the regret upper bound for UCB-NORMAL.

\begin{restatable}{prop}{noCVs}
	\label{cor:noCVs}
	\label{alg:UCB-Normal}
    Let reward and control variates have a multivariate normal distribution, $q$ be the number of control variates, $\Delta_i = \mu_{i^\star} - \mu_i$ be the sub-optimality gap for arm $i \in [K]$, and $C_{T,i}^{(\alpha)} = \EE{\left({\cV_{T,N_i(T)}^{(\alpha)}}/{\cV_{T,T}^{(\alpha)}}\right)^2}$, where $\cV_{t,s}^{(\alpha)}$ is the $100(1-1/t^\alpha)^{\text{th}}$ percentile of the $t-$distribution critical value with $s-1$ degrees of freedom. Then the regret of UCB-NORMAL with $\alpha=2$ in $T$ rounds is given by 
	\eqs{
		\Regret_T \le \sum_{i \ne i^\star} \hspace{-0.5mm} \left( \frac{4(\cV_{T,T}^{(\alpha)})^2 \sigma_i^2 C_{T,i}^{(\alpha)}}{\Delta_i^2} + \frac{\Delta_i\pi^2}{3} + \Delta_i\right).
	}
\end{restatable}

\noindent
\textbf{Proof sketch:}
The proof follows similar steps as the proof of Theorem 1 in \cite{NeurIPS21_verma2021stochastic} except the variance of mean reward estimator does not have any CVs and definition of $\cV_{T,T}^{(\alpha)}$ is different and hence the definition of $C_{T,i}^{(\alpha)}$.

\begin{rem}
    For $T=20000$, $q \le 20$ and $N_i(T) > 50$, we have $\frac{\cV_{T,N_i(T)}^{(2)}}{\cV_{T,T}^{(2)}}\leq 1.41$ (see Fig.~\ref{fig:C_Tiq1}) and can have one possible explicit regret bound by replacing $C_{T,i}^{(1)}$ with $2$. However, getting an explicit bound is still hard for a general case.
\end{rem}

\vspace{-3mm}
\begin{figure}[!ht]
	\centering
	\includegraphics[width=0.65\linewidth]{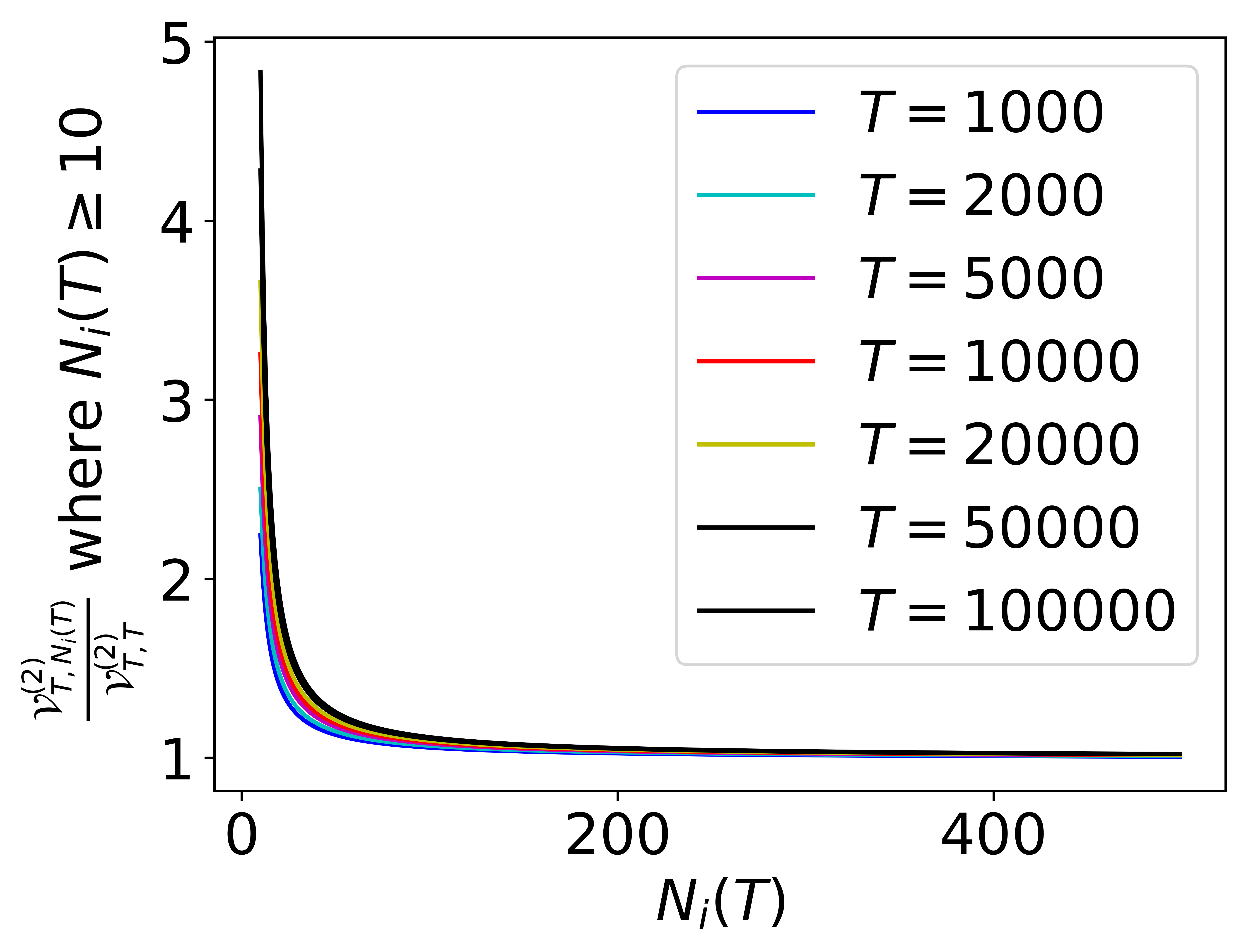}
	\caption{Variation in $\frac{\cV_{T,N_i(T)}^{(2)}}{\cV_{T,T}^{(2)}}$ with $S_i(T)$.}
	\label{fig:C_Tiq1}
\end{figure}

\noindent
\textbf{Comparing UCB with the standard bandit algorithm:}
Since the $t-$distribution critical values do not have an explicit form, it is hard to directly compare the UCB index (having $\cV_{T,T}^{(\alpha)}$ term) of \ref{alg:UCB-LCV} with that of UCB1-NORMAL and kl-UCB (having $\sqrt{\log (T)}$ term). However, we can empirically verify that $\left(\cV_{T,T}^{(\alpha)}\right)^2$ is upper bounded by $3.726 \log (T)$ for any $T > 1$ and $\alpha \le 2$ (see Fig.~\ref{fig:fig1a}). 

\vspace{-3mm}
\begin{figure}[!ht]
	\centering
    \includegraphics[width=0.65\linewidth]{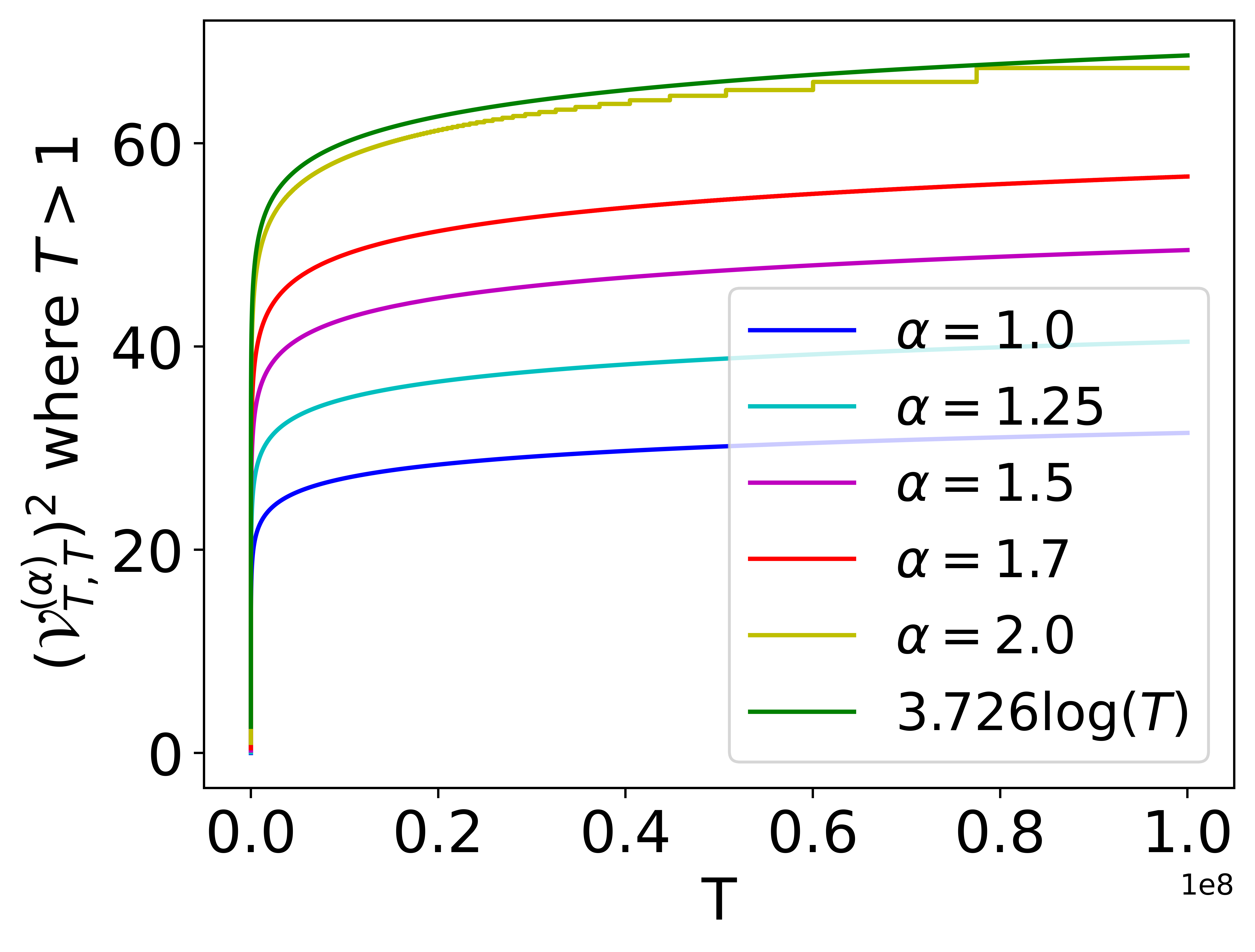}
    \caption{\vspace{-3mm} $\cV_{T,T}^{(\alpha)}$ vs. $\log(T)$.}
    \label{fig:fig1a}
\end{figure}

\subsection{CVs for all rewards samples}
For this case, we assume that the CVs are observed for each reward sample. This MAB-LCV problem is the same as the MAB-CV problem  \cite{NeurIPS21_verma2021stochastic}. When all reward samples have associated CVs, then the value of $\hat\lambda_{t,i}$ will be $0$ for all $t\ge 1$ and $i \in [K]$, and $M_i(t)$ will be the same as $S_i(t)$. The variant of \ref{alg:UCB-LCV} that is used for such MAB-LCV problems exactly behaves the same as the UCB-CV algorithm of \cite{NeurIPS21_verma2021stochastic}.

\begin{rem}[{\bf Regret Analysis of \ref{alg:UCB-LCV}:}] 
   For a fixed value of $\hat\lambda_{t,i}$, the variance estimator defined in Eq.~\ref{equ:varEst} is unbiased. However, it is still an open problem to find an exact unbiased sample variance estimator of the mean reward estimator (as defined in Lemma~\ref{lem:rewardEst}) when $\hat\lambda_{t,i}$ is random. As $\hat\lambda_{t,i}$ is also random in our setting, we cannot give any finite time regret guarantee for \ref{alg:UCB-LCV} and leave it for future work.
\end{rem}

	\section{General Distribution}
	\label{sec:no_dist}

We now consider the general case with no distributional assumptions on reward and associated CVs of arms. In this case also, we can continue the sets $\cH_{i}(t)$ and $\cH_{i}^c(t)$, but the estimator in Eq.~\ref{equ:estMeanRewardCV} obtained using the samples from  $\cH_{i}^c(t)$ need not remain an unbiased estimator. Therefore, we cannot use the $t$-distribution properties to obtain confidence intervals. However, we can use re-sampling methods such as jackknifing, splitting, and batching (see \cite{NeurIPS21_verma2021stochastic} for more details) to reduce the bias of the estimators and develop confidence intervals. Below, we briefly describe how the jackknifing and splitting methods can be applied in \cref{alg:UCB-LCV}.

In Jackknifing (\cite{Book82_efron1982jackknife, TAS83_efron1983leisurely}), $\hat\mu_{N_i(t),i}^{(nc)}$ is obtained as in Eq.~\ref{equ:estMeanReward} using the samples from $\cH_{i}(t)$ for each arm $i \in [K]$. To obtain the estimator using the samples from $\cH_{i}^c(t)$, we compute $M_i(t)=|\cH_{i}^c(t)|$ number of estimators for arm $i$, where $j^{\text{th}} \in [M_i(t)]$ estimator is computed using Eq.~\ref{equ:estMeanRewardCV} after leaving out $j^{\text{th}}$ reward-control variate sample pair from $\cH_{i}^c(t)$. All $M_i(t)$ estimators are averaged to obtain $\hat\mu_{M_i(t),i}^{(c)}$ which is then combined with $\hat\mu_{N_i(t),i}^{(nc)}$ as in Eq.~\ref{equ:estCvxMeanReward}. The sample variance of this new estimator is used to define the UCB of the arms in Eq.~\ref{equ:UCB}. 
In Splitting \cite{OR90_nelson1990control}, for each arm $i$, we define new samples $\bar{X}_{n,i}=X_{n,i} +\hat{\beta}^{\star -n}_{M_i(t),i}(\omega_i-W_{n,i})$ for all $n \in \cH^c_i(t)$, where $\hat{\beta}^{\star -n}_{M_i(t),i}$ is obtained as in Eq.~\ref{equ:BetaEst} after leaving out $n$th reward-control variate sample pair from $ \cH^c_i(t)$. These samples are used to obtain the estimate  $\hat\mu_{N_i(t),i}^{(nc)}=\sum_{n=1}^{M_i(t)}\bar{X}_{n,i}/M_i(t)$. This mean reward estimator is then linearly combined with $\hat\mu_{N_i(t),i}^{(nc)}$ to obtain a new estimator whose sample variance can be used to define the UCB of the arms in Eq.~\ref{equ:UCB}.

    \section{Experiments}
	\label{sec:experiments}

\begin{figure*}
	\centering
	\subfloat[Instance 1]{\label{fig:3a}
		\includegraphics[width=0.25\linewidth]{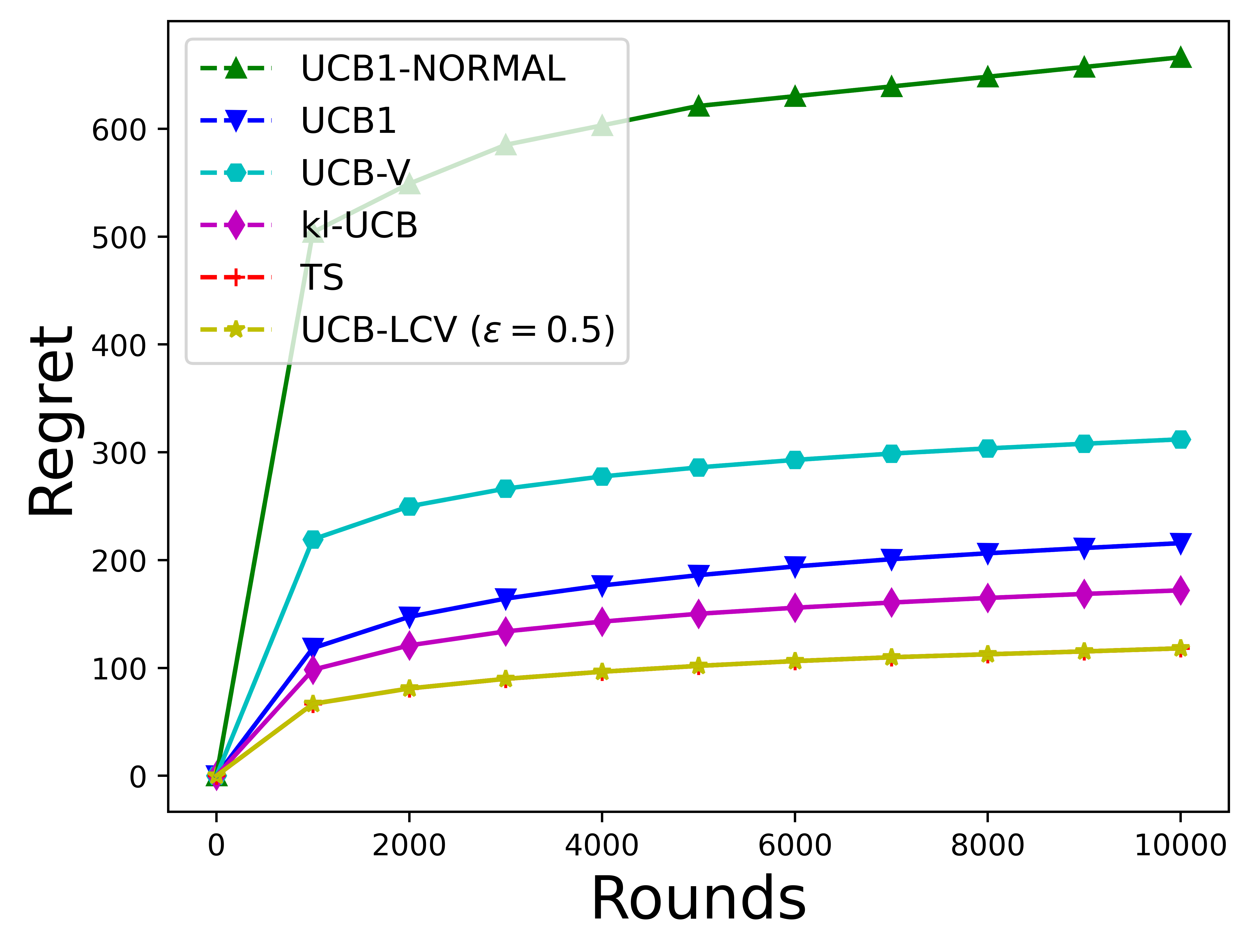}}
	\subfloat[Instance 2]{\label{fig:3b}
		\includegraphics[width=0.25\linewidth]{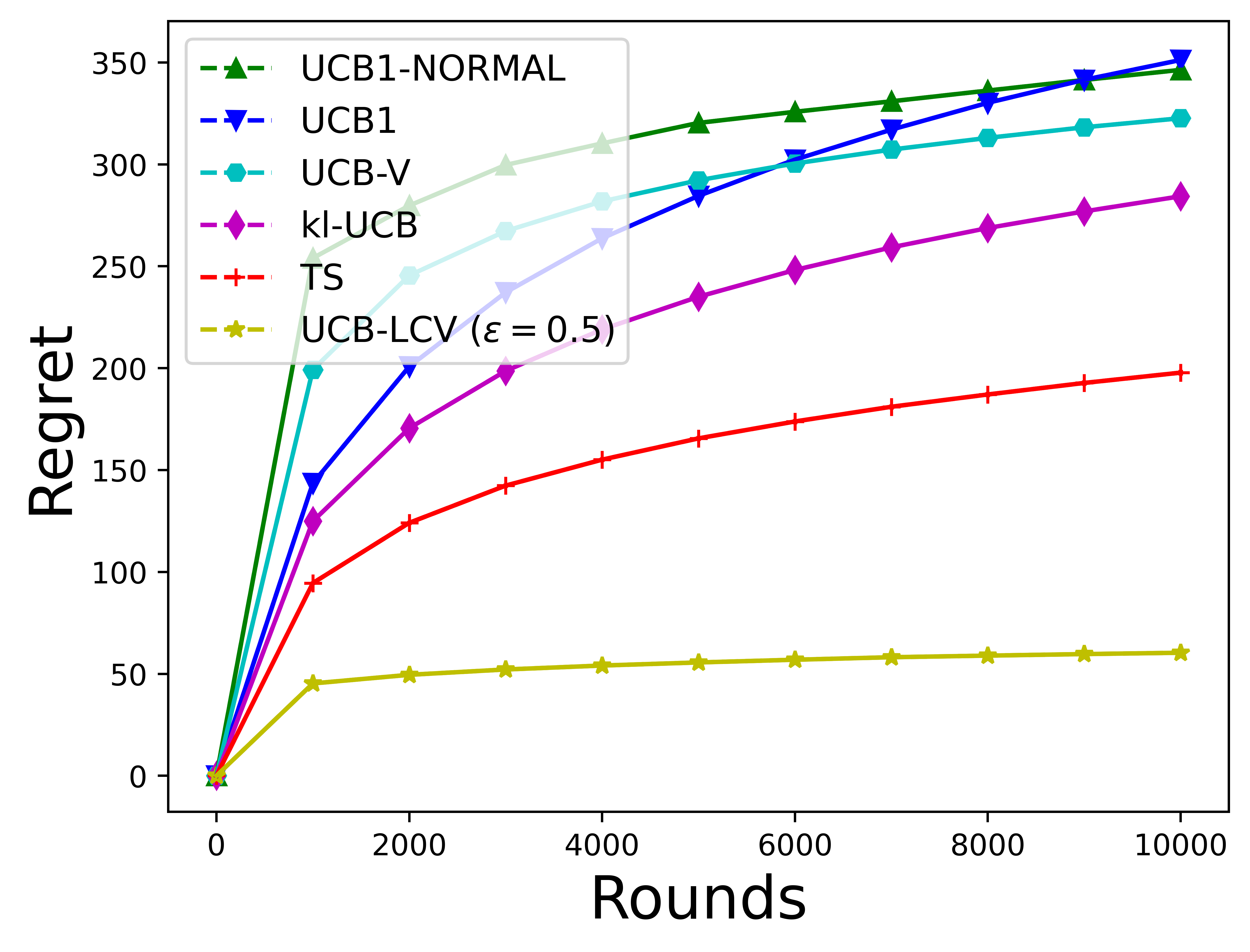}}
    \subfloat[Instance 3]{\label{fig:3c}
		\includegraphics[width=0.25\linewidth]{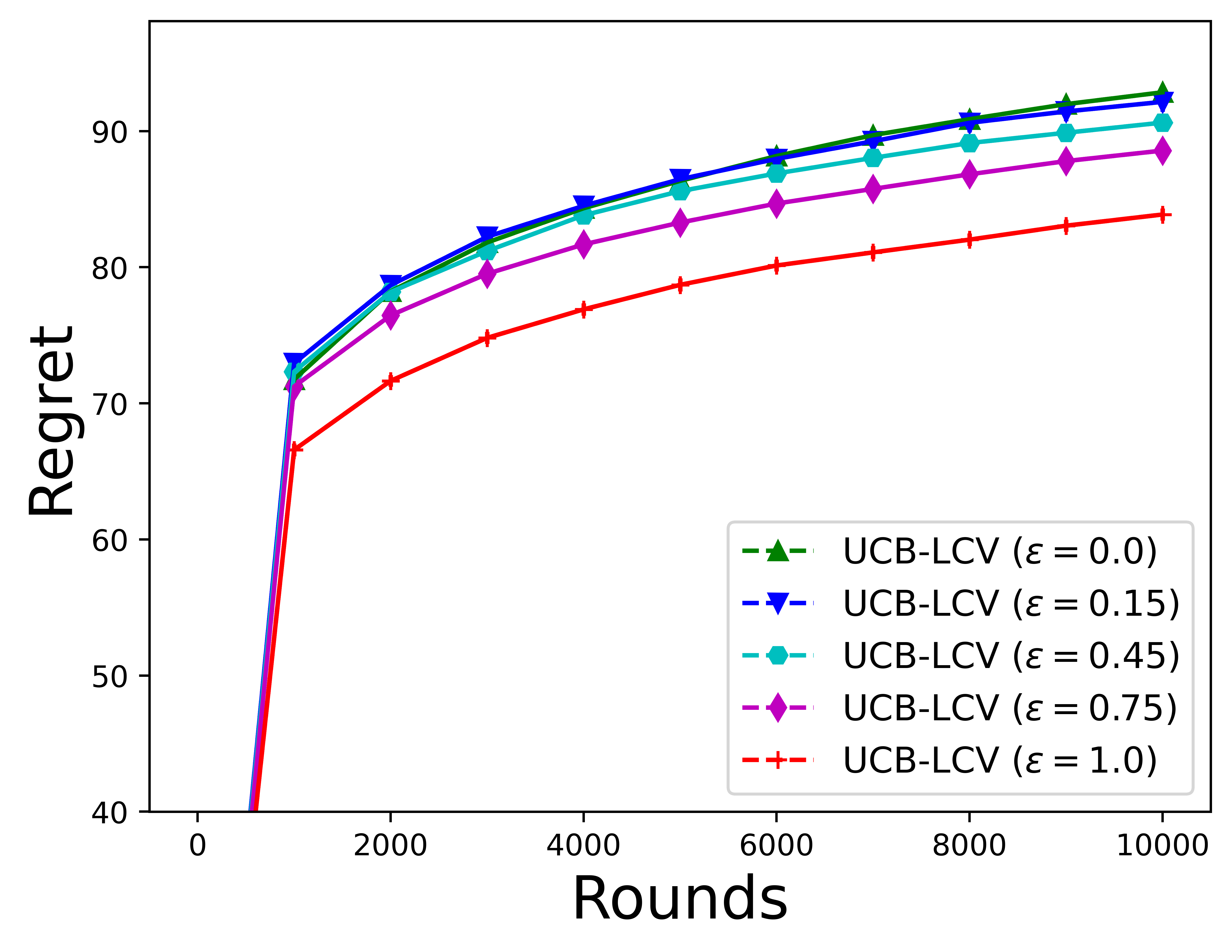}}
	\subfloat[Instance 4]{\label{fig:3d}
		\includegraphics[width=0.25\linewidth]{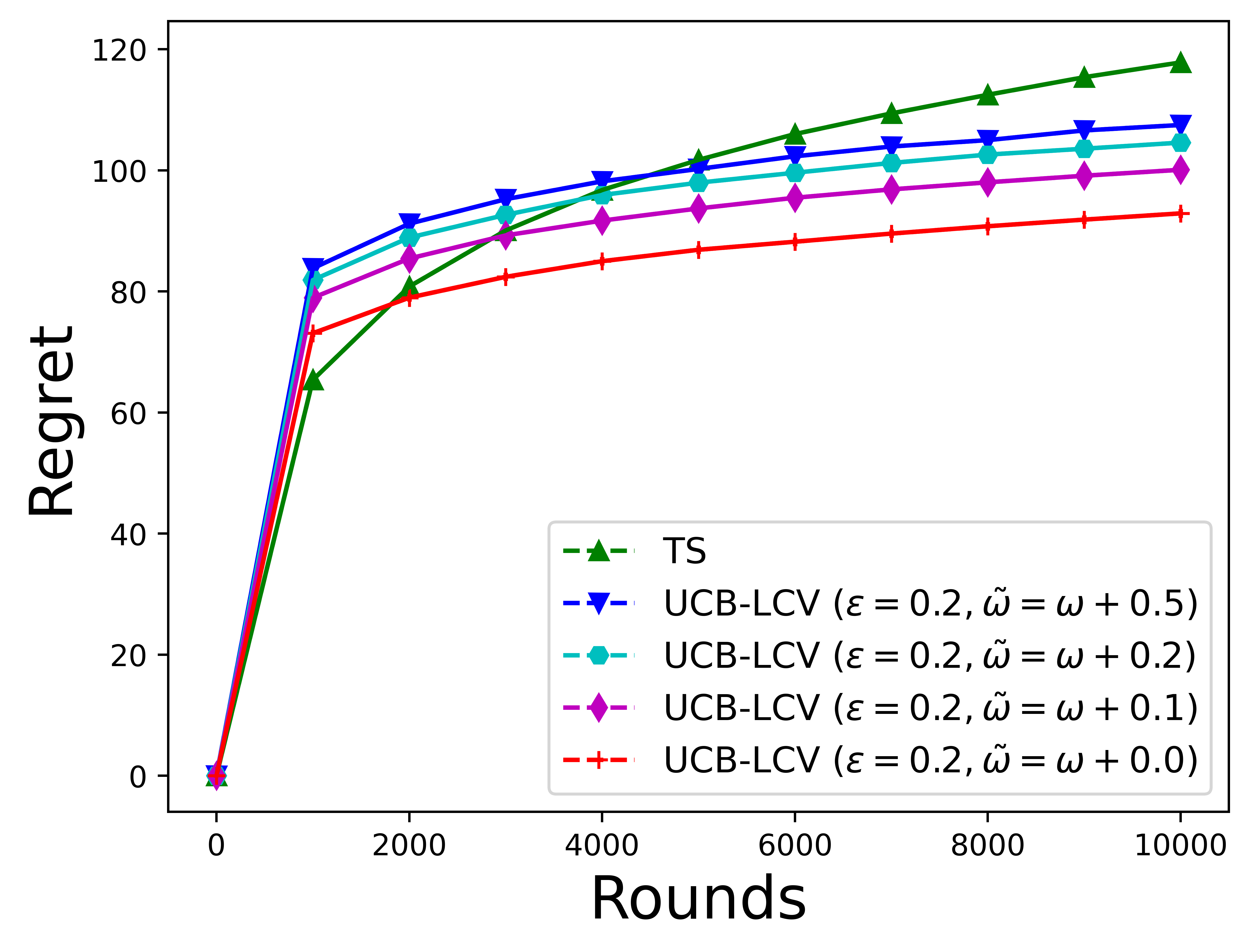}}
	\caption{Regret comparison of \ref{alg:UCB-LCV} with existing bandit algorithms (Fig.~\ref{fig:3a} and Fig.~\ref{fig:3b}).  Regret of \ref{alg:UCB-LCV} vs. the availability of control variate (Fig.~\ref{fig:3c}). Regret of \ref{alg:UCB-LCV} vs. erroneous mean of control variate (Fig.~\ref{fig:3d}). Since the confidence intervals are very small compared to the scale of the y-axis, they are not observable in the plots.}
	\label{fig:3}
\end{figure*}

\begin{figure*}[!ht]
    \vspace{-2mm}
	\centering
	\subfloat[Normal Distribution]{\label{fig:samplemods_a}
		\includegraphics[width=0.3\linewidth]{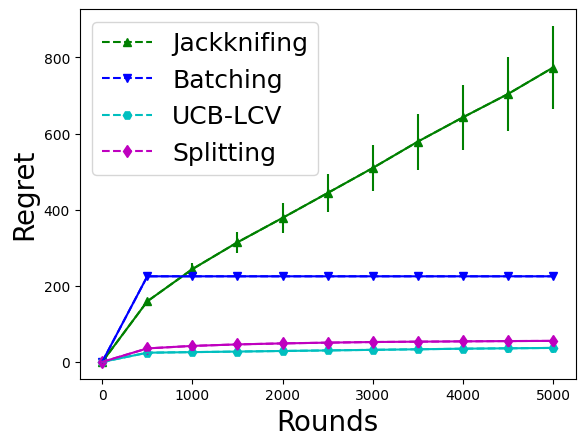}}
	\quad
	\subfloat[Multi-Modal Distribution]{\label{fig:samplemods_b}
		\includegraphics[width=0.3\linewidth]{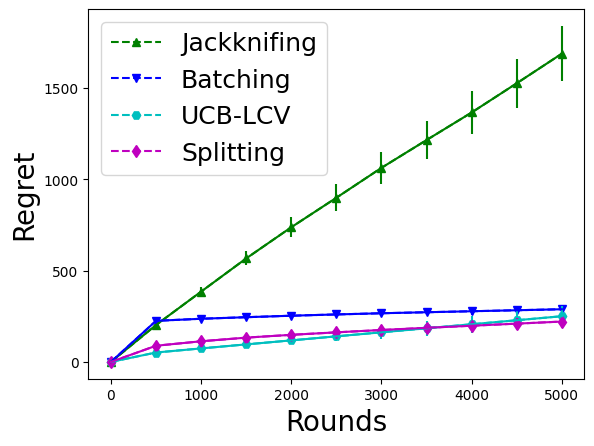}}
	\quad
	\subfloat[Log-normal distribution]{\label{fig:samplemods_c}
		\includegraphics[width=0.3\linewidth]{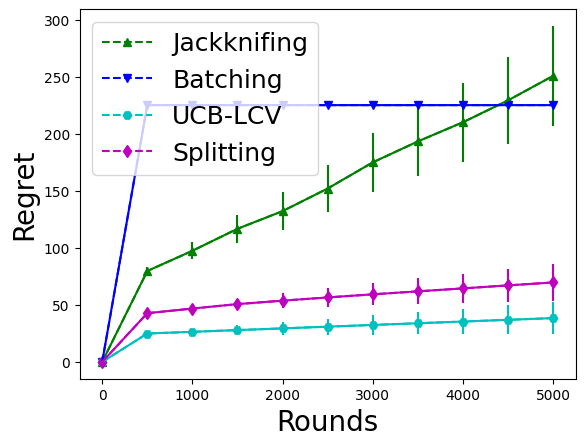}}
	\caption{Regret comparison of different variants of \ref{alg:UCB-LCV} that are based on Jackknifing, Splitting, and Batching methods for MAB-LCV problems with arms having a non-Gaussian distribution of reward and associated CVs. In these experiments, \ref{alg:UCB-LCV} assumes a Gaussian distribution irrespective of the underlying true distribution.}
	\label{fig:samplemods}
    \vspace{-2mm}
\end{figure*}

We empirically evaluate the different performance aspects of the proposed algorithms on synthetic problem instances. The results shown have been averaged over $100$ runs, and the regret bound is shown with a $95\%$ confidence interval (the vertical line on each curve shows the confidence interval). 

\subsection{Gaussian Distribution}
We compare the regret of \ref{alg:UCB-LCV} with UCB1 \cite{ML02_auer2002finite}, UCB1-NORMAL \cite{ML02_auer2002finite}, kl-UCB \cite{AS13_cappe2013kullback}, UCB-V \cite{TCS09_audibert2009exploration}, and Thompson Sampling \cite{AISTATS13_agrawal2013further} on synthetically generated problem instances with arms having Gaussian distributions for rewards and their associated CVs.  For each instance, we use $q=1$, $\alpha=2$, and $\epsilon = 0.5$, where $\epsilon$ denotes the fraction of rewards samples having control variate. Details are as follows:

\noindent
{\it Instance 1:} The reward and control variate in this experiment have a multivariate Gaussian distribution. The reward of each arm has two components. We treat one of the components as a control variate and consider its mean value to be known. In round $t$, the reward of arm $i$ is given by $X_{t,i} = V_{t,i} + Y_{t,i}$, where $V_{t,i} \sim \cN(\mu_{v,i}, \sigma_{v,i}^2)$ and $Y_{t,i} \sim \cN(\mu_{w,i}, \sigma_{w,i}^2)$. Hence, $X_{t,i} \sim \cN(\mu_{v,i}+\mu_{w,i}, \sigma_{v,i}^2 + \sigma_{w,i}^2)$. 
We treat $Y_{i}$ as a control variate associated with $X_i$. It can be easily shown that the correlation coefficient of $X_i$ and $Y_i$ is $\rho_i = \sqrt{\sigma_{w,i}^2/(\sigma_{v,i}^2+\sigma_{w,i}^2)}$. 
For an arm $i$, we set $\mu_{v,i} = 0.1*i$, $\mu_{w,i} = 0.1*i, \forall i \in [K]$. The value of $\sigma_{v,i}^2 = 0.01$ and $\sigma_{w,i}^2 = 0.01$ for all arms. The number of arms ($K$) is set to $10$ in all problem instances.

\noindent
{\it Instance 2:} The problem instance is the same as {\it Instance 1} except each arm has a different mean for the control variate. The mean value for each arm $i$ is set as $\mu_{w,i} = 0.5$. 

\noindent
{\it Instance 3:} This problem instance is also same as {\it Instance 1} except the fraction of rewards samples with control variate ($\epsilon$) varies as $[0.0, 0.15, 0.45, 0.75, 1.0]$. 

\noindent
{\it Instance 4:} We again use the same problem instance as {\it Instance 1} except the mean of the control variate is erroneous (as the error is set to be $[0.0, 0.1, 0.2, 0.5]$. 

\vspace{1mm}
\noindent
\textbf{Regret comparison:}
As expected, \ref{alg:UCB-LCV} outperforms the existing stochastic bandit algorithms for bandit problems with arms having a Gaussian distribution due to exploitation of the available CVs as shown in Fig.~\ref{fig:3a} and \ref{fig:3b}.

\vspace{1mm}
\noindent
\textbf{Variation of regret with different values of $\epsilon$:} 
In our next result, we capture how the amount of control variate influences regret, i.e., the variation in the regret when changing the number of control variate samples available. We consider Instance 3 for this experiment. We assume control variate information is given randomly, and we model it as a Bernoulli distribution with parameter $\epsilon$. The variation in regret while varying the probability of obtaining a control variate ($\epsilon$) with the reward sample in each round is shown in Fig.~\ref{fig:3c}. The regret comes down with the increase in the value of $\epsilon$, showing that having additional control variates always lowers the regret.

\vspace{1mm}
\noindent
\textbf{Variation of regret with estimated mean of CV:} 
We assume that the true mean of CV is known in all our experiments and theoretical results. This assumption may not be practical in many real-life applications. Therefore, we experimentally verify the effect of having only an erroneous mean of the control variate on regret. As shown in Fig.~\ref{fig:3d}, the regret increases with the error and even performs poorly than the existing bandit algorithm for significantly large errors.

\subsection{General Distribution}
We consider the Jackknifing, Splitting, and Batching methods for various distributions as shown in Fig.~\ref{fig:samplemods}. The Bernoulli parameter to determine the availability of the control variate is set to $0.2$. Using the same notation as Instance  1, $\mu_{v,i}$ is set to $0.6 - (i-1)*0.01$ and $\mu_{w,i}$ is $0.8 - (i-1)*0.01$. The distribution considered for Fig.~\ref{fig:samplemods_a} is a Gaussian distribution as described for Instance 1. For Fig.~\ref{fig:samplemods_b}, a multi-nodal distribution is generated by adding or subtracting $0.5$ from the reward with equal probability. Fig.~\ref{fig:samplemods_c} uses log-normal distributions for sampling the control variate, and the reward is generated by combining the log-normal sample with another log-normal sample. The mean and variance for the log-normal distributions are the same as the ones considered for Fig.~\ref{fig:samplemods_b}.

\vspace{1mm}
\noindent
\textbf{Performance of Jackknifing, Splitting, and Batching methods:} We compare the regret of \ref{alg:UCB-LCV} variants that are based on Jackknifing, Splitting, and Batching methods with the vanilla \ref{alg:UCB-LCV} that assumes a Gaussian distribution for rewards and its CVs irrespective of the underlying true distribution. The regret of the batching method is worse for all the instances. Whereas the jackknifing and splitting perform well for heavy-tailed distributions and \ref{alg:UCB-LCV} only performs better for the normal distribution, as shown in Fig.~\ref{fig:samplemods_a}.

	\section{Conclusion}
	\label{sec:conclusion}

In this work, we studied stochastic multi-armed bandits with side information where side information is limited and available in the form of control variates. We developed a method to combine the reward samples with and without control variates effectively. We showed that by appropriately choosing the weight, the resultant new estimator has a smaller variance compared to the case when control variates are not used. We have then empirically demonstrated that our proposed algorithm outperforms existing bandit algorithms for similar settings.
Interestingly, we demonstrated that for the case when no control variates are available, the performance of our algorithm, named UCB-NORMAL, is better than that of UCB1-NORMAL introduced by \cite{ML02_auer2002finite}. On the other extreme, when control variates are available in each round, our algorithm is the same as the UCB-CV algorithm in \cite{NeurIPS21_verma2021stochastic}. 
Proving a regret upper bound of \cref{alg:UCB-LCV} remains an open problem. 
We also assume that the means of the control variates are known. Therefore, an interesting future direction is to see how the performance gets affected when the means are not known precisely or estimated. 
Other promising future directions include extending limited CVs to parameterized bandits \cite{NeurIPS23_verma2024exploiting} and using CVs in contextual bandits with preference feedback (\cite{NeurIPS21_verma2021stochastic,ICLR25_verma2025neural, verma2025active}), with their applications in LLMs such as prompt optimization \cite{arXiv24_lin2024prompt} and active alignment \cite{lin2025activedpo}.

    \section*{Acknowledgements}
    Manjesh K. Hanawal acknowledges funding support from MeitY under the MeitY-NSF Research Collaboration scheme and DST under the DST-INRIA Targeted Programme and DST-RSF cooperation programme.

	\bibliographystyle{IEEEtran}
	\bibliography{ref}

\end{document}